\documentclass[letterpaper]{article} 
\usepackage{aaai24}  
\usepackage{times}  
\usepackage{helvet}  
\usepackage{courier}  
\usepackage[hyphens]{url}  
\usepackage{graphicx} 
\usepackage{natbib}  
\usepackage{caption} 
\usepackage{algorithm}
\usepackage{algorithmic}

\usepackage{newfloat}
\usepackage{listings}
\DeclareCaptionStyle{ruled}{labelfont=normalfont,labelsep=colon,strut=off} 
\floatstyle{ruled}
\newfloat{listing}{tb}{lst}{}
\floatname{listing}{Listing}
\usepackage{amsmath}
\usepackage{amssymb}
\usepackage{amsthm}
\usepackage{bm}
\usepackage{booktabs}
\usepackage{mathtools}
\usepackage{multirow}
\usepackage{subcaption}
\usepackage{tikz}
\newtheorem{proposition}{Proposition}
\newtheorem*{lemma*}{Lemma}

\DeclarePairedDelimiter{\norm}{\lVert}{\rVert}

\title{Efficient Conditional Diffusion Model with Probability Flow Sampling for Image Super-resolution}
\author{
    Yutao Yuan, Chun Yuan
}
\affiliations{
    Tsinghua University\\
    yuanyt21@mails.tsinghua.edu.cn, yuanc@sz.tsinghua.edu.cn
}

\begin{document}

\maketitle

\begin{abstract}
Image super-resolution is a fundamentally ill-posed problem because multiple valid high-resolution images exist for one low-resolution image.
Super-resolution methods based on diffusion probabilistic models can deal with the ill-posed nature by learning the distribution of high-resolution images conditioned on low-resolution images, avoiding the problem of blurry images in PSNR-oriented methods.
However, existing diffusion-based super-resolution methods have high time consumption with the use of iterative sampling, while the quality and consistency of generated images are less than ideal due to problems like color shifting.
In this paper, we propose Efficient Conditional Diffusion Model with Probability Flow Sampling (ECDP) for image super-resolution.
To reduce the time consumption, we design a continuous-time conditional diffusion model for image super-resolution, which enables the use of probability flow sampling for efficient generation.
Additionally, to improve the consistency of generated images, we propose a hybrid parametrization for the denoiser network, which interpolates between the data-predicting parametrization and the noise-predicting parametrization for different noise scales.
Moreover, we design an image quality loss as a complement to the score matching loss of diffusion models, further improving the consistency and quality of super-resolution.
Extensive experiments on DIV2K, ImageNet, and CelebA demonstrate that our method achieves higher super-resolution quality than existing diffusion-based image super-resolution methods while having lower time consumption.
Our code is available at https://github.com/Yuan-Yutao/ECDP.
\end{abstract}

\section{Introduction}

Image super-resolution, the task of recovering high-resolution (HR) images from low-resolution (LR) images, is fundamentally an ill-posed problem. Given an LR image, there are more than one HR images consistent with the input. Existing PSNR-oriented super-resolution methods~\cite{srcnn,edsr,rcan} that learn deterministic mappings from LR images to HR images using pixel losses are effectively predicting the mean of all plausible HR images, and tend to generate blurry HR images with unsatisfactory visual quality. Super-resolution methods based on generative models deal with the ill-posed nature by learning the distribution of HR images conditioned on LR images, allowing for the generation of multiple diverse results from a single input image and avoiding the problem of blurry images.

Recently, the use of diffusion probabilistic models~\cite{ddpm,scoresde}, a trending class of generative models, have grown popular in image super-resolution.
SR3~\cite{sr3} and SRDiff~\cite{srdiff} adapts Diffusion Denoising Probabilistic Models (DDPMs)~\cite{ddpm} for image super-resolution.
They define a Markovian forward process that gradually adds Gaussian noise into image data, and use denoiser neural networks conditioned on LR images to learn its reverse process and generate new images from noise.
They are able to generate diverse and realistic HR images with fine details.

However, there are still challenging aspects that remain to be improved for diffusion-based super-resolution.
Diffusion models typically generate new images iteratively using a Markov chain, which necessitates many neural network evaluations and makes the super-resolution process \textbf{time-consuming}. Additionally, they are prone to problems like color shifting, making the \textbf{quality and consistency} of generated images less than ideal and reducing their performance on super-resolution.

To tackle the challenges, we propose \underline{E}fficient \underline{C}onditional \underline{D}iffusion Model with \underline{P}robability Flow Sampling (ECDP) for image super-resolution.
It gradually corrupts HR images using stochastic differential equations (SDEs), and learns to restore the original images with a denoiser network conditioned on LR images. We generate super-resolution images using probability flow sampling, which can be performed with low time consumption using ordinary differential equation (ODE) solvers.
Additionally, to improve the consistency of generated images with LR input, we use a hybrid parametrization in the denoiser network. It uses the $x_0$-parametrization that predicts the clean data directly in addition to the commonly used $\epsilon$-parametrization, and smoothly interpolates between them for different noise scales.
Moreover, we introduce an image quality loss as a complement to the score matching loss of diffusion models. It measures the feature-space distance of generated HR images and ground-truth images in the dataset, further improving the consistency and quality of super-resolution.
Extensive experiments on multiple datasets encompassing face super-resolution and general super-resolution demonstrate the effectiveness of our approach.

Our main contributions are summarized as follows:
\begin{itemize}
    \item We propose Efficient Conditional Diffusion Model with Probability Flow Sampling (ECDP) for image super-resolution, which generates realistic super-resolution images with low time costs.
    \item With a continuous-time conditional diffusion model based on SDEs designed for super-resolution, we can generate super-resolution images using probability flow sampling, which reduces the time consumption of super-resolution.
    \item We propose score matching with hybrid-parametrization and design an image quality loss for diffusion-based image super-resolution, improving the consistency and quality of generated images.
    \item Extensive experiments on DIV2K, ImageNet, and CelebA demonstrate that our method achieves higher super-resolution quality than existing diffusion-based image super-resolution methods while having lower time consumption.
\end{itemize}

\section{Related Work}

\paragraph{Diffusion probabilistic models}
Diffusion probabilistic models are a family of deep generative models with great success in image generation. DDPM~\cite{ddpm} defines a Markovian diffusion process on image data, gradually adding noise into the image, and learns to reproduce the original image with a sequence of denoisers. It achieves impressive high-quality image generation results. Various improvements to the model have been proposed, including improved architectures~\cite{guided-diffusion} loss reweighting~\cite{improved-ddpm}, and fast sampling~\cite{ddim}. DDPMs are shown to be equivalent to denoising score matching over multiple noise levels~\cite{multidenoise}, which are unified and generalized to the continuous case with SDEs~\cite{scoresde}. We build our super-resolution method on top of the SDE formulation by extending it and conditioning on LR images.

Besides image generation, diffusion models have been applied to a large range of tasks in computer vision. A popular strategy among them is to formulate the task as a conditional generation problem, using diffusion models to predict the distribution of outputs conditioned on the inputs. This strategy has achieved success in text-to-image generation~\cite{imagen}, image super-resolution~\cite{sr3,srdiff}, image inpainting~\cite{palette}, and image colorization~\cite{palette}, among others. A different line of research focuses on using existing diffusion models in a zero-shot manner. By taking an unconditional diffusion model and enforcing consistency with reference images during sampling, it is possible to perform image editing~\cite{sdedit,ilvr} and image inpainting~\cite{repaint} without task-specific training. More recently, this approach has been generalized for a family of linear and non-linear inverse problems~\cite{ddrm,ddnm,dps}.

\paragraph{Image super-resolution}
A lot of super-resolution methods based on deep learning have been proposed in recent years. Most of the early work takes a regression-based approach~\cite{srcnn,edsr,rcan}, learning a deterministic one-to-one mapping from LR images to HR images with L2 or L1 losses. Since the posterior distribution of HR images is highly multi-modal, these methods tend to generate blurry images, effectively predicting the mean of the distribution. To improve the visual quality of generated images, GAN-based approaches~\cite{srgan,esrgan} are proposed for super-resolution. They are able to generate HR images with high quality, but tend to suffer from mode collapse, difficult optimization and low consistency with LR images. Normalizing flows have also been used for image super-resolution~\cite{srflow,hcflow}. They are able to estimate the distribution of HR images conditioned on LR images, allowing for diverse and realistic image super-resolution. However, normalizing flows require invertible architectures, limiting the expressiveness of the models.

Recently, several super-resolution methods using diffusion models have been proposed. SR3~\cite{sr3} and SRDiff~\cite{srdiff} adapts DDPM~\cite{ddpm} for super-resolution, making the model conditional on LR images. They are able to generate realistic, high quality HR images. SR3 concatenates upscaled LR images to noisy HR images as the input to the denoiser network, making the model conditional on LR images. SRDiff uses an LR encoder to extract features from LR input, and further uses residual prediction to improve the convergence speed and performance of the model.
Some methods~\cite{ccdf,dps,gdp,diffpir} use a different zero-shot approach as opposed to aforementioned supervised methods, taking an unconditional diffusion model trained on image generation and modifying its sampling process with guidance. However, their performance is often less than ideal compared to supervised methods due to lack of dedicated training.

\section{Preliminaries}

\begin{figure*}[t]
    \centering
    \includegraphics[width=\linewidth]{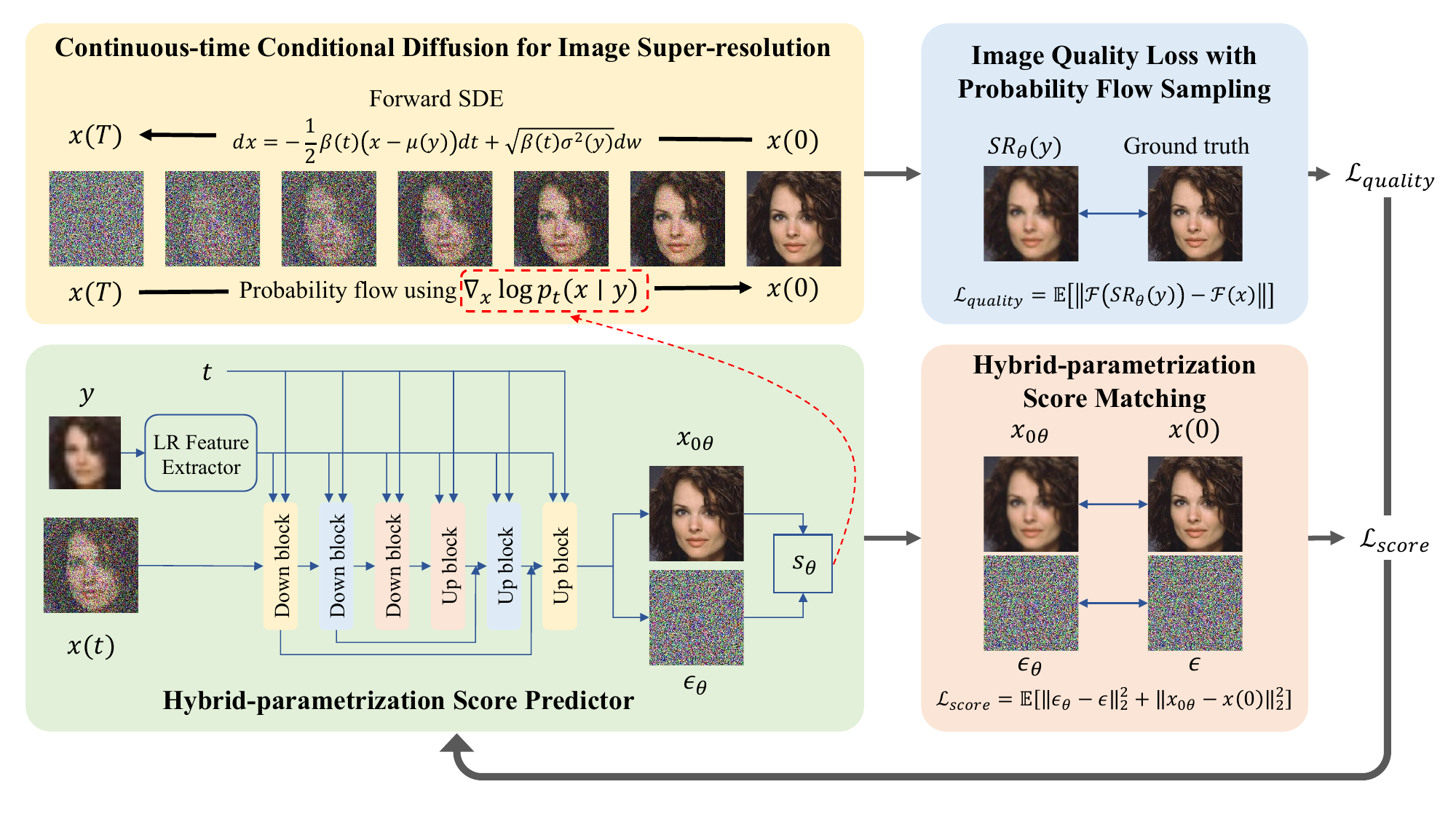}
    \caption{Overview of ECDP. \emph{Top left}: Continuous-time conditional diffusion uses a forward SDE to transform images into noise, and generate new images from noise using probability flow. \emph{Bottom}: The conditional score in the probability flow is approximated with a hybrid-parametrization score predictor $\bm{s}_\theta$, which is trained using score matching. \emph{Top right}: An additional image quality loss that compares the generated HR images with the ground truth is computed using probability flow sampling, improving the quality of super-resolution results.}
    \label{fig:framework}
\end{figure*}

Given a dataset $\mathcal{X} = \{\bm{x}_i\}$ that follows an unknown distribution $p(x)$, continuous-time diffusion probabilistic models define a forward process that gradually injects noise into data $\bm{x}$. This process can be described as the solution to a SDE running from $t = 0$ to $t = T$, starting with i.i.d. samples $\bm{x}(0)$ from the dataset~\cite{scoresde}:
\begin{equation}
    \mathrm{d} \bm{x} = \bm{f}(\bm{x}, t) \mathrm{d} t + g(t) \mathrm{d} \bm{w}
    \label{eq:general-forward}
\end{equation}
where $\bm{f}(\bm{x}, t)$ and $g(t)$ are predefined functions, and $\bm{w}$ is Brownian motion. Denote by $\bm{x}(t)$ the solution at time $t$ and $p_t(\bm{x})$ its probability distribution. The parameters of the forward process are chosen so $p_T$ ends up as a prior distribution with tractable sampling.

It has be shown~\cite{scoresde} that the data distribution $p_0$ can be recovered from $p_T$ using another reverse SDE, running backwards from $t = T$ to $t = 0$:
\begin{equation}
    \mathrm{d} \bm{x} = \left[ \bm{f}(\bm{x}, t) - g(t)^2 \nabla_{\bm{x}} \log p_t(\bm{x}) \right] \mathrm{d} t + g(t) \mathrm{d} \bm{w}
    \label{eq:general-reverse}
\end{equation}

The term $\nabla_{\bm{x}} \log p_t(\bm{x})$ is the score of the distribution $p_t(\bm{x})$, which is intractable because the data distribution is unknown. Diffusion models learn the data distribution by approximating $\nabla_{\bm{x}} \log p_t(\bm{x})$ with a score prediction network $\bm{s}_{\theta}(\bm{x}, t)$. Due to the intractability of the marginal distribution $p_t(\bm{x})$, score matching techniques are deployed in training, giving rise to the score matching loss:
\begin{equation}
    \mathcal{L} = \mathbb{E}_{\bm{x},t} \mathbb{E}_{\bm{x}(t)} \left[ \norm*{\bm{s}_{\theta}(\bm{x}(t), t) - \nabla_{\bm{x}(t)} \log p_{0t}(\bm{x}(t) \mid \bm{x})}_2^2 \right]
    \label{eq:general-scoreloss}
\end{equation}
where $p_{0t}$ is the transition probability of the forward process and can often be computed analytically. In practice the loss is often reweighted, where terms associated with different $t$ are assigned different weights, to ensure better convergence in training.

Besides the reverse SDE, it is also possible to sample from the learned distribution using probability flow, which takes the form of an ODE:
\begin{equation}
    \mathrm{d} \bm{x} = \left[ \bm{f}(\bm{x}, t) - \frac{1}{2} g(t)^2 \nabla_{\bm{x}} \log p_t(\bm{x}) \right] \mathrm{d} t
    \label{eq:general-ode}
\end{equation}
It can be solved using off-the-shelf ODE solvers, allowing efficient sampling using much less time than the reverse SDE.

\section{Proposed Method}

In this section, we present Efficient Conditional Diffusion Model with Probability Flow Sampling (ECDP) for image super-resolution.
From a dataset of paired HR and LR images $\mathcal{D} = \{(\bm{x}_i, \bm{y}_i)\}$, our model learns the conditional distribution $p(\bm{x} \mid \bm{y})$. Given LR images, super-resolution images are generated by sampling from this conditional distribution.
Our method is illustrated in Figure \ref{fig:framework}.

\subsection{Continuous-Time Conditional Diffusion for Image Super-Resolution}
In our method, we design a conditional diffusion model for image super-resolution with a forward SDE that transforms HR images $\bm{x}$ into noise while conditioned on LR images $\bm{y}$:
\begin{equation}
    \mathrm{d} \bm{x} = -\frac{1}{2} \beta(t) (\bm{x} - \bm{\mu}(\bm{y})) \mathrm{d} t + \sqrt{\beta(t) \sigma^2(\bm{y})} \mathrm{d} \bm{w}
    \label{eq:ours-forward}
\end{equation}
where $\bm{\mu}(\bm{y})$ and $\sigma^2(\bm{y})$ are the per-pixel mean and variance of $p(\bm{x} \mid \bm{y})$ respectively, and $\beta(t)$ is a hyperparameter that controls how fast noise is injected into data.

Since it is impossible to compute $\bm{\mu}(\bm{y})$ and $\sigma^2(\bm{y})$ without direct access to the true data distribution, we approximate $\bm{\mu}(\bm{y})$ by upscaling $\bm{y}$ using bicubic interpolation, and set $\sigma^2(\bm{y})$ to a predefined constant determined using the empirical variance of $\bm{x} - \bm{\mu}(\bm{y})$ over the training dataset. It is worth noting that the use of $\bm{\mu}(\bm{y})$ in the forward process is similar to residual prediction~\cite{srdiff}, which subtracts upscaled LR images from HR images before diffusion. We differ from residual prediction in that we integrate the LR images to the forward process of diffusion directly, and we additionally considers the variance of HR images.

With the use of $\bm{\mu}(\bm{y})$ and $\sigma^2(\bm{y})$, the forward process has the mean-preserving and variance-preserving properties, as formalized below:
\begin{proposition}
    \label{thm:ours-diffusion}
    The forward process given by \eqref{eq:ours-forward} keeps the mean and variance of $\bm{x}(t)$ conditioned on $\bm{y}$ unchanged during the transform from $t = 0$ to $t = T$. More specifically: 
    \begin{align}
        \mathbb{E}[\bm{x}(t) \mid \bm{y}] &= \bm{\mu}(\bm{y}) \\
        \text{Var}[\bm{x}(t) \mid \bm{y}] &= \sigma^2(\bm{y})
    \end{align}
\end{proposition}
By maintaining the mean and variance of data during the forward process, the amount of change in the data distribution is minimized, making model training easier and image generation faster.

To enable conditional generation of HR images, we learn to approximate the conditional score $\nabla_{\bm{x}} \log p_t(\bm{x} \mid \bm{y})$ with a conditional score prediction network $\bm{s}_{\theta}(\bm{x}, \bm{y}, t)$. The score prediction network is trained with denoising score matching~\cite{denoise-sm} using the following loss:
\begin{multline}
    \mathcal{L}_{\text{score}} = \mathbb{E}_{(\bm{x}, \bm{y}) \sim \mathcal{D}} \mathbb{E}_{t} \mathbb{E}_{\bm{\epsilon} \sim \mathcal{N}(\bm{0}, \bm{I})} \\
    \left[ \norm*{\bm{s}_{\theta}\left(\hat{\bm{\mu}}(\bm{x}, \bm{y}, t) + \sqrt{\hat{\sigma}^2(\bm{y}, t)} \bm{\epsilon}, \bm{y}, t \right) + \frac{\bm{\epsilon}}{\sqrt{\hat{\sigma}^2(\bm{y}, t)}}}_2^2 \right]
    \label{eq:scoreloss}
\end{multline}
where
\begin{align}
    \hat{\bm{\mu}}(\bm{x}, \bm{y}, t) &= \sqrt{\alpha(t)} (\bm{x} - \bm{\mu}(\bm{y})) + \bm{\mu}(\bm{y}) \\
    \hat{\sigma}^2(\bm{y}, t) &= (1 - \alpha(t)) \sigma^2(\bm{y}) \\
    \alpha(t) &= \exp\left( -\int_{0}^{t} \beta(s) \mathrm{d} s \right)
\end{align}

For efficient generation of super-resolution images, we sample HR images with probability flow, which can be much faster than SDE-based sampling due to its deterministic nature.
Given an LR image $\bm{y}$, new HR images can be sampled from the learned conditional distribution by solving the following ODE from $t = T$ to $t = 0$, starting with $\bm{x}(T)$ sampled from the limit distribution of the forward SDE:
\begin{gather}
    \bm{x}(T) \sim \mathcal{N}(\bm{\mu}(\bm{y}), \sigma^2(\bm{y}) I) \\
    \mathrm{d} \bm{x} = \left[ -\frac{1}{2} \beta(t) (\bm{x} - \bm{\mu}(\bm{y})) - \frac{1}{2} \beta(t) \sigma^2(\bm{y}) \bm{s}_{\theta}(\bm{x}, \bm{y}, t) \right] \mathrm{d} t
    \label{eq:ours-ode}
\end{gather}

\subsection{Hybrid-Parametrization Score Matching}

\begin{figure}
    \centering
    \begin{subcaptionblock}{0.3\linewidth}
        \includegraphics[width=\linewidth]{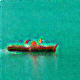}
        \caption{$\epsilon$-param}
        \label{fig:param-error-eps}
    \end{subcaptionblock}
    \hspace{0.02\linewidth}
    \begin{subcaptionblock}{0.3\linewidth}
        \includegraphics[width=\linewidth]{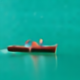}
        \caption{$x_0$-param}
        \label{fig:param-error-x0}
    \end{subcaptionblock}
    \hspace{0.02\linewidth}
    \begin{subcaptionblock}{0.3\linewidth}
        \includegraphics[width=\linewidth]{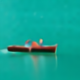}
        \caption{hybrid-param}
    \end{subcaptionblock}
    \caption{The denoised images produced by the $\epsilon$-parametrization, the $x_0$-parametrization and the hybrid-parametrization.}
    \label{fig:param-error}
\end{figure}

Existing diffusion probabilistic models~\cite{ddpm,scoresde} typically parametrize the denoiser network in such a way that its output is matched against the normalized noise component $\bm{\epsilon}$ in noisy data $\bm{x}(t)$ during training, so the network output has identical variance across different $t$. We denote it the $\epsilon$-parametrization. $\epsilon$-parametrization has been proved effective in image generation~\cite{ddpm}, improving the quality of denoising predictions when the amount of noise is low. However, $\epsilon$-parametrization alone does not work well for image super-resolution. In image super-resolution, the space of plausible HR images $\bm{x}$ is strongly constrained by the consistency with paired LR images $\bm{y}$ and concentrated around a single point. $\epsilon$-parametrization has difficulty learning this consistency constraint, because it can only recover the clean HR images indirectly by subtracting the predicted noise from noisy images, requiring wildly varying prediction values for different noisy data. It tends to produce very inconsistent denoising results when the amount of noise is large. This effect can be seen in Figure \ref{fig:param-error-eps}, where $\epsilon$-parametrization does not produce satisfactory denoising predictions, leaving artifacts in the denoised image. As a result, super-resolution methods using the $\epsilon$-parametrization tends to produce HR images inconsistent with LR input.

A natural alternative to $\epsilon$-parametrization is to predict the clean data component in noisy data as opposed to the noise component. We denote it the $x_0$-parametrization. $x_0$-parametrization has been investigated in the context of unconditional image generation~\cite{dualoutput,diff-distill}, where it was found to have better performance than $\epsilon$-parametrization in certain cases. In image super-resolution, $x_0$-parametrization produces clean images with no artifacts even when the amount of noise is large, because the denoiser network can use the LR inputs to recover a good estimate of clean HR images directly.

For the conditional image diffusion in our method, $\epsilon$-parametrization and $x_0$-parametrization are defined by expressing the score prediction values in terms of neural networks $\bm{\epsilon}_{\theta}$ and $\bm{x}_{0\theta}$ respectively:
\begin{align}
    \bm{s}_{\theta,\epsilon}(\bm{x}, \bm{y}, t) &= -\frac{1}{\sqrt{\hat{\sigma}^2(\bm{y}, t)}} \bm{\epsilon}_{\theta}(\bm{x}, \bm{y}, t) \\
    \bm{s}_{\theta,x_0}(\bm{x}, \bm{y}, t) &= -\frac{1}{\hat{\sigma}^2(\bm{y})} \left( \bm{x} - \hat{\bm{\mu}}(\bm{x}_{0\theta}(\bm{x}, \bm{y}, t), \bm{y}, t) \right)
\end{align}
where $\bm{\epsilon}_{\theta}$ and $\bm{x}_{0\theta}$ are trained using the following reweighted version of the score matching loss \eqref{eq:scoreloss}:
\begin{gather}
\begin{split}
    \mathcal{L}_{\text{score}} &= \mathbb{E}_{(\bm{x}, \bm{y}) \sim \mathcal{D}} \mathbb{E}_{t} \mathbb{E}_{\bm{\epsilon} \sim \mathcal{N}(\bm{0}, \bm{I})} \\
    &\mathrel{\phantom{=}} \quad \left[ \norm*{\bm{\epsilon}_{\theta}(\bm{x}_t, \bm{y}, t) - \bm{\epsilon}}_2^2 + \norm*{\bm{x}_{0\theta}(\bm{x}_t, \bm{y}, t) - \bm{x}}_2^2 \right]
\end{split}
    \shortintertext{where}
    \bm{x}_t = \hat{\bm{\mu}}(\bm{x}, \bm{y}, t) + \sqrt{\hat{\sigma}^2(\bm{y}, t)} \bm{\epsilon}
\end{gather}

As the $\epsilon$-parametrization has low estimation errors in the low-noise region, while $x_0$-parametrization has low estimation errors in the high-noise region, to take advantage of both parametrizations, we use a hybrid approach to smoothly interpolate between these two parametrizations:
\begin{equation}
    \bm{s}_{\theta}(\bm{x}, \bm{y}, t) = \lambda(t) \bm{s}_{\theta,\epsilon}(\bm{x}, \bm{y}, t) + (1 - \lambda(t)) \bm{s}_{\theta,x_0}(\bm{x}, \bm{y}, t)
    \label{eq:hybrid-param}
\end{equation}
where $\lambda(t)$ is the interpolation coefficient. In our method, we choose $\lambda(t) = \alpha(t)^c$, where $c$ is a constant in the range $[0.5, 1.5]$, selected for each dataset individually to minimize the prediction error for the hybrid parametrization. The hybrid-parametrization is able to achieve low estimation errors in all noise scales.

\subsection{Image Quality Loss with Probability Flow Sampling}
Diffusion probabilistic models do not directly optimize for the quality of generated images during training. Instead, the quality is only optimized indirectly in the learning of data distribution with score matching loss.
In image super-resolution, each LR image has only one paired HR image in the dataset, and it is difficult for diffusion models to estimate the conditional distribution of HR images accurately from this single data point.
It is therefore important for diffusion-based super-resolution methods to add additional image quality losses into training.

In CNN-based super-resolution methods, losses like the pixel loss and the perceptual loss directly measure the distance between super-resolution images and the ground truth. However, these losses have not received usage in diffusion-based super-resolution methods, as they require generating new HR images during training, which is computationally expensive for diffusion models using stochastic sampling.

With the use of probability flow sampling in our method for efficient super-resolution, it becomes possible to introduce such losses into the training of the diffusion model.
Therefore, we propose an image quality loss for diffusion-based image super-resolution, defined as the feature-space distance between generated HR images and the ground truth:
\begin{equation}
    \mathcal{L}_{\text{quality}} = \mathbb{E}_{(\bm{x}, \bm{y}) \sim \mathcal{D}} \left[ \norm*{\mathcal{F}(\text{SR}_{\theta}(\bm{y})) - \mathcal{F}(\bm{x})} \right]
\end{equation}
where $\text{SR}_{\theta}(y)$ is an HR image sample generated using the probability flow.
$\mathcal{F}$ is chosen as the feature maps of a VGG network pretrained on image classification, making $\mathcal{L}_{\text{quality}}$ equivalent to the perceptual loss in CNN-based methods, but in principle it can be any function that converts images to feature vectors.

To compute gradients of the image quality loss with regards to the network parameters, it is necessary to backpropagate through $\text{SR}_{\theta}(y)$, the solution of the probability flow ODE.
This can be achieved efficiently using the adjoint method~\cite{neuralode}, which expresses the gradient of ODE solutions with regards to model parameters in terms of another augmented ODE, making it possible to compute gradients without depending on the intermediate values of the original ODE. Compared with direct backpropagation through the ODE solver, the memory consumption of computing the image quality loss is reduced from $O(s)$ to $O(1)$, where $s$ is the number of sampling steps.

\begin{figure}
    \centering
    \subcaptionbox{no $\mathcal{L}_{\text{quality}}$}
    {\begin{tikzpicture}[every node/.style={inner sep=0,outer sep=0}]
        \node {\includegraphics[width=0.45\linewidth]{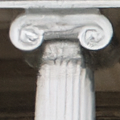}};
        \draw[draw=red] (0.07\linewidth, 0.04\linewidth) rectangle ++(0.11\linewidth, 0.11\linewidth);
        \draw[draw=red] (-0.08\linewidth, -0.05\linewidth) rectangle ++(0.14\linewidth, 0.10\linewidth);
    \end{tikzpicture}}
    \hspace{0.03\linewidth}
    \subcaptionbox{with $\mathcal{L}_{\text{quality}}$}
    {\begin{tikzpicture}[every node/.style={inner sep=0,outer sep=0}]
        \node {\includegraphics[width=0.45\linewidth]{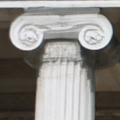}};
        \draw[draw=red] (0.07\linewidth, 0.04\linewidth) rectangle ++(0.11\linewidth, 0.11\linewidth);
        \draw[draw=red] (-0.08\linewidth, -0.05\linewidth) rectangle ++(0.14\linewidth, 0.10\linewidth);
    \end{tikzpicture}}
    \caption{Visualization of images generated by our method trained without and with $\mathcal{L}_{\text{quality}}$. The image generated by the model with $\mathcal{L}_{\text{quality}}$ has more visible structure (the lines on the pillar) and less background noise.}
    \label{fig:image-quality-loss-effect}
\end{figure}

As seen in Figure \ref{fig:image-quality-loss-effect}, the image quality loss can significantly improve the quality of super-resolution images for diffusion-based methods. The model trained with $\mathcal{L}_{\text{quality}}$ is able to generate cleaner images with more visible structure and less background noise compared to the model trained without $\mathcal{L}_{\text{quality}}$.

\section{Experiments}

In this section, we conduct experiments on multiple datasets encompassing both face images and general images to demonstrate the performance of our method.

\subsection{Experimental Setup}

\paragraph{Model architecture and hyperparameters}
In our experiments, we set $\beta(t)$ in our forward process \eqref{eq:ours-forward} to a linear function increasing from $\beta(0) = 0.1$ and $\beta(T) = 20$, matching the settings in VP-SDE~\cite{scoresde}. Score prediction from $\epsilon$-parametrization and $x_0$-parametrization are produced using a single network with two output heads, using an U-Net architecture with BigGAN~\cite{biggan} residual blocks. To make the model conditional on LR images, we use an LR feature extractor with the RRDB~\cite{esrgan} architecture and add the features to each layer of the score prediction network.
To improve the training efficiency, we use only the standard score matching loss in early stages of training, and then add our image quality loss later on. To generate HR images, we perform probability flow sampling using a standard Runge-Kutta ODE solver with absolute and relative error tolerance of $10^{-4}$.

\paragraph{Datasets}
For general image super-resolution ($4 \times$), we evaluate the performance of various methods on two datasets, DIV2K~\cite{div2k} and ImageNet~\cite{imagenet}. For DIV2K, models are trained using random HR image crops of size $160 \times 160$ and evaluated with full-size images. The training dataset is augmented with images from Flickr2K following the practice in earlier methods~\cite{srflow,srdiff}. For ImageNet, images are center cropped and resized to $256 \times 256$ for HR following SR3~\cite{sr3}, and further downsampled to $64 \times 64$ with bicubic interpolation for LR.

For face image super-resolution ($8 \times$), we train and evaluate on CelebA. Images are cropped and resized to $160 \times 160$ following the procedures in ~\cite{srflow}, and then downsampled as LR using bicubic interpolation.

\paragraph{Baselines}
We compare our image super-resolution method with the following baselines: PSNR-oriented method RRDB~\cite{esrgan}, GAN-based method ESRGAN~\cite{esrgan}, normalizing flow-based methods SRFlow~\cite{srflow} and HCFlow~\cite{hcflow}, as well as several diffusion-based methods, SR3~\cite{sr3}, SRDiff~\cite{srdiff}, IR-SDE~\cite{irsde}, GDP~\cite{gdp}, and DiffPIR~\cite{diffpir}. Among diffusion-based methods, GDP and DiffPIR uses pretrained unconditional diffusion models, while other methods including ours train conditional models for super-resolution. We use official results for baselines where available, and train the models from scratch otherwise.

\paragraph{Evaluation metrics}
In addition to PSNR and SSIM, two standard metrics for image super-resolution, we use \textbf{LPIPS}~\cite{lpips} to measure the visual quality of super-resolution results. It is a perceptual metric that is known to correlate better with human perception than traditional metrics like PSNR and SSIM, and is therefore considered the main metric in our experiments.

\subsection{Results of Image Super-Resolution}

\paragraph{General image super-resolution}
The quantitative results on DIV2K and ImageNet are shown in Table \ref{tab:results-df2k} and \ref{tab:results-imagenet} respectively, and the visual results are shown in Figure \ref{fig:results-general}. On both datasets, our method outperforms all baseline methods in LPIPS and achieves the best overall super-resolution quality.
RRDB generates blurry images and has high LPIPS distance despite achieving high PSNR and SSIM scores, further confirming that PSNR and SSIM does not correlate well with perceptual quality. SR3 has good PSNR and SSIM metrics on ImageNet, but its performance is poor in the main metric LPIPS. GDP and DiffPIR, the two diffusion-based methods that use pretrained unconditional models as opposed to training dedicated conditional models, have inferior performance in the three metrics compared to other diffusion-based methods, demonstrating the importance of designing diffusion models specifically for image super-resolution.

\begin{table}
    \centering
    \begin{tabular}{ccccc}
        \toprule
        Method & LPIPS\textdownarrow & PSNR\textuparrow & SSIM\textuparrow & \#Params \\
        \midrule
        RRDB & 0.253 & 29.44 & 0.84 & 16.7M \\
        ESRGAN & 0.124 & 26.22 & 0.75 & 16.7M \\
        SRFlow & 0.120 & 27.09 & 0.76 & 39.5M \\
        HCFlow & 0.110 & 26.61 & 0.74 & 23.2M \\
        \midrule
        SR3 & 0.175 & 25.90 & 0.75 & 97.8M \\
        SRDiff & 0.136 & 27.41 & \textbf{0.79} & 37.6M \\
        IR-SDE & 0.231 & 25.90 & 0.66 & 137.1M \\
        Ours & \textbf{0.108} & \textbf{28.03} & \textbf{0.79} & 42.6M \\
        \bottomrule
    \end{tabular}
    \caption{Results of $4\times$ image super-resolution on DIV2K. Diffusion-based methods and non-diffusion-based methods are grouped together respectively. Best results among diffusion-based methods are highlighted in bold.}
    \label{tab:results-df2k}
\end{table}

\begin{table}
    \centering
    \begin{tabular}{cccccc}
        \toprule
        Method & LPIPS\textdownarrow & PSNR\textuparrow & SSIM\textuparrow & \#Params \\
        \midrule
        RRDB & 0.245 & 27.23 & 0.78 & 16.7M \\
        ESRGAN & 0.123 & 24.18 & 0.67 & 16.7M \\
        SRFlow & 0.142 & 24.09 & 0.67 & 39.5M \\
        HCFlow & 0.129 & 25.07 & 0.70 & 23.2M \\
        \midrule
        SR3 & 0.191 & \textbf{26.40} & \textbf{0.76} & 625M \\
        SRDiff & 0.154 & 24.04 & 0.59 & 37.6M \\
        GDP & 0.240 & 24.42 & 0.68 & --- \\
        DiffPIR & 0.219 & 25.19 & 0.70 & --- \\
        Ours & \textbf{0.110} & 25.81 & 0.74 & 37.6M \\
        \bottomrule
    \end{tabular}
    \caption{Results of $64 \times 64 \rightarrow 256 \times 256$ image super-resolution on ImageNet. Best results among diffusion-based methods are highlighted in bold. \#Params for GDP and DiffPIR are omitted because they use pretrained unconditional diffusion models and do not train new models on their own.}
    \label{tab:results-imagenet}
\end{table}

\begin{table}
    \centering
    \begin{tabular}{c|ccc}
        \toprule
        Method & SR3 & SRDiff & IR-SDE \\
        Inference Time & 83.1s & 2.4s & 6.2s \\
        \midrule
        Method & GDP & DiffPIR & Ours \\
        Inference Time & 94.6s & 13.0s & \textbf{2.0s} \\
        \bottomrule
    \end{tabular}
    \caption{Time required to generate a single $256 \times 256$ HR image for diffusion-based methods.}
    \label{tab:infer-time}
\end{table}

\begin{table}
    \centering
    \begin{tabular}{ccccc}
        \toprule
        Method & LPIPS\textdownarrow & PSNR\textuparrow & SSIM\textuparrow & \#Params \\
        \midrule
        RRDB & 0.230 & 26.59 & 0.77 & 16.7M \\
        ESRGAN & 0.120 & 22.88 & 0.63 & 16.7M \\
        SRFlow & 0.110 & 25.24 & 0.71 & 40.0M \\
        HCFlow & 0.090 & 24.83 & 0.69 & 27.0M \\
        \midrule
        SR3 & 0.109 & 24.26 & 0.68 & 97.8M \\ 
        SRDiff & 0.106 & 25.38 & \textbf{0.74} & 12.0M \\
        Ours & \textbf{0.097} & \textbf{25.48} & 0.73 & 40.0M \\
        \bottomrule
    \end{tabular}
    \caption{Results of face image super-resolution on CelebA. Best results among diffusion-based methods are highlighted in bold.}
    \label{tab:results-face}
\end{table}

\begin{figure*}[t]
    \centering
    \begin{tabular}{@{\hspace{0pt}}c@{\hspace{12pt}}c@{\hspace{12pt}}c@{\hspace{12pt}}c@{\hspace{12pt}}c@{\hspace{12pt}}c@{\hspace{0pt}}}
        \multirow{4}{*}[0.106\linewidth]{\includegraphics[width=0.272\linewidth]{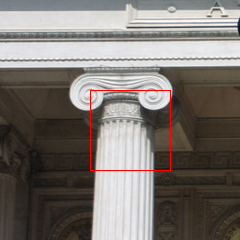}}
        & \includegraphics[width=0.122\linewidth]{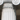}
        & \includegraphics[width=0.122\linewidth]{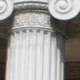}
        & \includegraphics[width=0.122\linewidth]{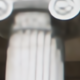}
        & \includegraphics[width=0.122\linewidth]{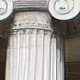}
        & \includegraphics[width=0.122\linewidth]{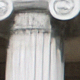} \\
        & LR & Ground Truth & RRDB & ESRGAN & SRFlow \\
        & \includegraphics[width=0.122\linewidth]{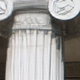}
        & \includegraphics[width=0.122\linewidth]{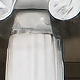}
        & \includegraphics[width=0.122\linewidth]{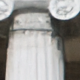}
        & \includegraphics[width=0.122\linewidth]{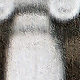}
        & \includegraphics[width=0.122\linewidth]{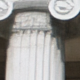} \\
        & HCFlow & SR3 & SRDiff & IR-SDE & Ours \\
        \multirow{4}{*}[0.106\linewidth]{\includegraphics[width=0.272\linewidth]{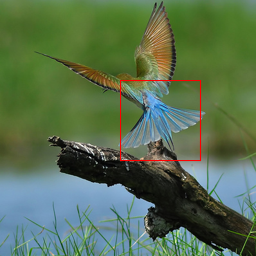}}
        & \includegraphics[width=0.122\linewidth]{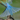}
        & \includegraphics[width=0.122\linewidth]{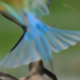}
        & \includegraphics[width=0.122\linewidth]{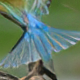}
        & \includegraphics[width=0.122\linewidth]{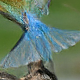}
        & \includegraphics[width=0.122\linewidth]{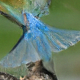} \\
        & LR & RRDB & ESRGAN & SRFlow & HCFlow \\
        & \includegraphics[width=0.122\linewidth]{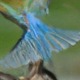}
        & \includegraphics[width=0.122\linewidth]{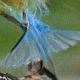}
        & \includegraphics[width=0.122\linewidth]{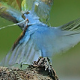}
        & \includegraphics[width=0.122\linewidth]{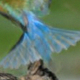}
        & \includegraphics[width=0.122\linewidth]{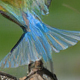} \\
        & SR3 & SRDiff & GDP & DiffPIR & Ours
    \end{tabular}
    \caption{Visual results of general image super-resolution. The first two rows are results on DIV2K validation set. The last two rows are results on ImageNet dev set.}
    \label{fig:results-general}
\end{figure*}

In addition, the sampling speed of diffusion-based super-resolution methods in our experiments is listed in Table \ref{tab:infer-time}. Our method is able to perform super-resolution with the least amount of time among all diffusion-based methods, confirming the efficiency of probability flow sampling.

\paragraph{Face image super-resolution}
The results on CelebA are shown in Table \ref{tab:results-face}. Similar to the general image super-resolution case, our method reaches state-of-the-art performance and has the best LPIPS score among diffusion-based methods.
Our method is able to produce faces with realistic details without generating unnecessary noise and distorting the images.

\subsection{Ablation Study}

\begin{table}
    \centering
    \begin{tabular}{c@{\hspace{8pt}}c@{\hspace{8pt}}c@{\hspace{9pt}}c@{\hspace{9pt}}c}
        \toprule
        Parametrization & $\mathcal{L}_{\text{quality}}$ & LPIPS\textdownarrow & PSNR\textuparrow & SSIM\textuparrow \\
        \midrule
        hybrid, $c = 1.0$ & \checkmark & 0.108 & 28.03 & 0.79 \\
        \midrule
        hybrid, $c = 0.5$ & \checkmark & 0.109 & 28.07 & 0.79 \\
        hybrid, $c = 1.5$ & \checkmark & 0.107 & 28.01 & 0.79 \\
        $\epsilon$-parametrization & \checkmark & 0.128 & 28.02 & 0.79 \\
        $x_0$-parametrization & \checkmark & 0.112 & 27.72 & 0.77 \\
        hybrid, $c = 1.0$ & \texttimes & 0.120 & 27.26 & 0.76 \\
        \bottomrule
    \end{tabular}
    \caption{Results of the ablation study. All results are measured on the DIV2K $4 \times$ task.}
    \label{tab:ablation}
\end{table}

To study the influence of difference choices of parametrizations in the denoiser network as well as the use of image quality loss, we conduct ablation studies as illustrated in Table \ref{tab:ablation}.
Our hybrid-parametrization achieves the best super-resolution results in all metrics among three choices of parametrizations since it has the advantages of both $\epsilon$ and $x_0$-parametrization. $\epsilon$-parametrization has the lowest LR-PSNR score among all, confirming our analysis that it generates images with low consistency. The performance of hybrid-parametrization is mostly insensitive to the choice of $c$.
For the image quality loss, it can be seen that the model trained with $\mathcal{L}_{\text{quality}}$ achieves significantly better metrics than the model without $\mathcal{L}_{\text{quality}}$, confirming its importance for diffusion-based super-resolution.

\section{Conclusion}
In this paper, we proposed ECDP, an image super-resolution framework with a continuous-time conditional diffusion model. It deploys a hybrid-parametrization denoiser network to learn the conditional score function, and generates super-resolution images efficiently using probability flow sampling. An additional image quality loss for diffusion-based super-resolution is introduced, which is computed efficiently and improves the quality of super-resolution results. Experiments demonstrate that our method achieves higher super-resolution quality than existing diffusion-based image super-resolution methods while having lower time consumption.

\section{Acknowledgments}

This work was supported by the National Key R\&D Program of China (2022YFB4701400/4701402), SSTIC Grant (JCYJ20190809172201639, WDZC20200820200655001), Shenzhen Key Laboratory (ZDSYS20210623092001004), and Beijing Key Lab of Networked Multimedia.

\bibliography{yuan}

\begin{thebibliography}{35}
\providecommand{\natexlab}[1]{#1}

\bibitem[{Agustsson and Timofte(2017)}]{div2k}
Agustsson, E.; and Timofte, R. 2017.
\newblock {NTIRE} 2017 Challenge on Single Image Super-Resolution: Dataset and Study.
\newblock In \emph{2017 {IEEE} Conference on Computer Vision and Pattern Recognition Workshops, {CVPR} Workshops 2017, Honolulu, HI, USA, July 21-26, 2017}, 1122--1131. {IEEE} Computer Society.

\bibitem[{Benny and Wolf(2022)}]{dualoutput}
Benny, Y.; and Wolf, L. 2022.
\newblock Dynamic Dual-Output Diffusion Models.
\newblock In \emph{{IEEE/CVF} Conference on Computer Vision and Pattern Recognition, {CVPR} 2022, New Orleans, LA, USA, June 18-24, 2022}, 11472--11481. {IEEE}.

\bibitem[{Brock, Donahue, and Simonyan(2019)}]{biggan}
Brock, A.; Donahue, J.; and Simonyan, K. 2019.
\newblock Large Scale {GAN} Training for High Fidelity Natural Image Synthesis.
\newblock In \emph{7th International Conference on Learning Representations, {ICLR} 2019, New Orleans, LA, USA, May 6-9, 2019}. OpenReview.net.

\bibitem[{Chen et~al.(2018)Chen, Rubanova, Bettencourt, and Duvenaud}]{neuralode}
Chen, T.~Q.; Rubanova, Y.; Bettencourt, J.; and Duvenaud, D. 2018.
\newblock Neural Ordinary Differential Equations.
\newblock In Bengio, S.; Wallach, H.~M.; Larochelle, H.; Grauman, K.; Cesa{-}Bianchi, N.; and Garnett, R., eds., \emph{Advances in Neural Information Processing Systems 31: Annual Conference on Neural Information Processing Systems 2018, NeurIPS 2018, December 3-8, 2018, Montr{\'{e}}al, Canada}, 6572--6583.

\bibitem[{Choi et~al.(2021)Choi, Kim, Jeong, Gwon, and Yoon}]{ilvr}
Choi, J.; Kim, S.; Jeong, Y.; Gwon, Y.; and Yoon, S. 2021.
\newblock {ILVR:} Conditioning Method for Denoising Diffusion Probabilistic Models.
\newblock In \emph{2021 {IEEE/CVF} International Conference on Computer Vision, {ICCV} 2021, Montreal, QC, Canada, October 10-17, 2021}, 14347--14356. {IEEE}.

\bibitem[{Chung et~al.(2023)Chung, Kim, McCann, Klasky, and Ye}]{dps}
Chung, H.; Kim, J.; McCann, M.~T.; Klasky, M.~L.; and Ye, J.~C. 2023.
\newblock Diffusion Posterior Sampling for General Noisy Inverse Problems.
\newblock In \emph{The Eleventh International Conference on Learning Representations, {ICLR} 2023, Kigali, Rwanda, May 1-5, 2023}. OpenReview.net.

\bibitem[{Chung, Sim, and Ye(2022)}]{ccdf}
Chung, H.; Sim, B.; and Ye, J.~C. 2022.
\newblock Come-Closer-Diffuse-Faster: Accelerating Conditional Diffusion Models for Inverse Problems through Stochastic Contraction.
\newblock In \emph{{IEEE/CVF} Conference on Computer Vision and Pattern Recognition, {CVPR} 2022, New Orleans, LA, USA, June 18-24, 2022}, 12403--12412. {IEEE}.

\bibitem[{Dhariwal and Nichol(2021)}]{guided-diffusion}
Dhariwal, P.; and Nichol, A.~Q. 2021.
\newblock Diffusion Models Beat GANs on Image Synthesis.
\newblock In Ranzato, M.; Beygelzimer, A.; Dauphin, Y.~N.; Liang, P.; and Vaughan, J.~W., eds., \emph{Advances in Neural Information Processing Systems 34: Annual Conference on Neural Information Processing Systems 2021, NeurIPS 2021, December 6-14, 2021, virtual}, 8780--8794.

\bibitem[{Dong et~al.(2014)Dong, Loy, He, and Tang}]{srcnn}
Dong, C.; Loy, C.~C.; He, K.; and Tang, X. 2014.
\newblock Learning a Deep Convolutional Network for Image Super-Resolution.
\newblock In Fleet, D.~J.; Pajdla, T.; Schiele, B.; and Tuytelaars, T., eds., \emph{European Conference on Computer Vision}, Lecture Notes in Computer Science, 184--199. Springer.

\bibitem[{Fei et~al.(2023)Fei, Lyu, Pan, Zhang, Yang, Luo, Zhang, and Dai}]{gdp}
Fei, B.; Lyu, Z.; Pan, L.; Zhang, J.; Yang, W.; Luo, T.; Zhang, B.; and Dai, B. 2023.
\newblock Generative Diffusion Prior for Unified Image Restoration and Enhancement.
\newblock In \emph{{IEEE/CVF} Conference on Computer Vision and Pattern Recognition, {CVPR} 2023, Vancouver, BC, Canada, June 17-24, 2023}, 9935--9946. {IEEE}.

\bibitem[{Ho, Jain, and Abbeel(2020)}]{ddpm}
Ho, J.; Jain, A.; and Abbeel, P. 2020.
\newblock Denoising Diffusion Probabilistic Models.
\newblock In Larochelle, H.; Ranzato, M.; Hadsell, R.; Balcan, M.; and Lin, H., eds., \emph{Advances in Neural Information Processing Systems 33: Annual Conference on Neural Information Processing Systems 2020, NeurIPS 2020, December 6-12, 2020, virtual}.

\bibitem[{Kawar et~al.(2022)Kawar, Elad, Ermon, and Song}]{ddrm}
Kawar, B.; Elad, M.; Ermon, S.; and Song, J. 2022.
\newblock Denoising Diffusion Restoration Models.
\newblock In \emph{NeurIPS}.

\bibitem[{Ledig et~al.(2017)Ledig, Theis, Huszar, Caballero, Cunningham, Acosta, Aitken, Tejani, Totz, Wang, and Shi}]{srgan}
Ledig, C.; Theis, L.; Huszar, F.; Caballero, J.; Cunningham, A.; Acosta, A.; Aitken, A.~P.; Tejani, A.; Totz, J.; Wang, Z.; and Shi, W. 2017.
\newblock Photo-Realistic Single Image Super-Resolution Using a Generative Adversarial Network.
\newblock In \emph{{IEEE} Conference on Computer Vision and Pattern Recognition}, 105--114. {IEEE} Computer Society.

\bibitem[{Li et~al.(2021)Li, Yang, Chang, Feng, Xu, Li, and Chen}]{srdiff}
Li, H.; Yang, Y.; Chang, M.; Feng, H.; Xu, Z.; Li, Q.; and Chen, Y. 2021.
\newblock SRDiff: Single Image Super-Resolution with Diffusion Probabilistic Models.
\newblock arXiv:2104.14951.

\bibitem[{Liang et~al.(2021)Liang, Lugmayr, Zhang, Danelljan, Gool, and Timofte}]{hcflow}
Liang, J.; Lugmayr, A.; Zhang, K.; Danelljan, M.; Gool, L.~V.; and Timofte, R. 2021.
\newblock Hierarchical Conditional Flow: {A} Unified Framework for Image Super-Resolution and Image Rescaling.
\newblock In \emph{2021 {IEEE/CVF} International Conference on Computer Vision, {ICCV} 2021, Montreal, QC, Canada, October 10-17, 2021}, 4056--4065. {IEEE}.

\bibitem[{Lim et~al.(2017)Lim, Son, Kim, Nah, and Lee}]{edsr}
Lim, B.; Son, S.; Kim, H.; Nah, S.; and Lee, K.~M. 2017.
\newblock Enhanced Deep Residual Networks for Single Image Super-Resolution.
\newblock In \emph{{IEEE} Conference on Computer Vision and Pattern Recognition Workshops}, 1132--1140. {IEEE} Computer Society.

\bibitem[{Lugmayr et~al.(2020)Lugmayr, Danelljan, Gool, and Timofte}]{srflow}
Lugmayr, A.; Danelljan, M.; Gool, L.~V.; and Timofte, R. 2020.
\newblock SRFlow: Learning the Super-Resolution Space with Normalizing Flow.
\newblock In Vedaldi, A.; Bischof, H.; Brox, T.; and Frahm, J., eds., \emph{Computer Vision - {ECCV} 2020 - 16th European Conference, Glasgow, UK, August 23-28, 2020, Proceedings, Part {V}}, volume 12350 of \emph{Lecture Notes in Computer Science}, 715--732. Springer.

\bibitem[{Lugmayr et~al.(2022)Lugmayr, Danelljan, Romero, Yu, Timofte, and Gool}]{repaint}
Lugmayr, A.; Danelljan, M.; Romero, A.; Yu, F.; Timofte, R.; and Gool, L.~V. 2022.
\newblock RePaint: Inpainting using Denoising Diffusion Probabilistic Models.
\newblock In \emph{{IEEE/CVF} Conference on Computer Vision and Pattern Recognition, {CVPR} 2022, New Orleans, LA, USA, June 18-24, 2022}, 11451--11461. {IEEE}.

\bibitem[{Luo et~al.(2023)Luo, Gustafsson, Zhao, Sj{\"{o}}lund, and Sch{\"{o}}n}]{irsde}
Luo, Z.; Gustafsson, F.~K.; Zhao, Z.; Sj{\"{o}}lund, J.; and Sch{\"{o}}n, T.~B. 2023.
\newblock Image Restoration with Mean-Reverting Stochastic Differential Equations.
\newblock In Krause, A.; Brunskill, E.; Cho, K.; Engelhardt, B.; Sabato, S.; and Scarlett, J., eds., \emph{International Conference on Machine Learning, {ICML} 2023, 23-29 July 2023, Honolulu, Hawaii, {USA}}, volume 202 of \emph{Proceedings of Machine Learning Research}, 23045--23066. {PMLR}.

\bibitem[{Meng et~al.(2022)Meng, He, Song, Song, Wu, Zhu, and Ermon}]{sdedit}
Meng, C.; He, Y.; Song, Y.; Song, J.; Wu, J.; Zhu, J.; and Ermon, S. 2022.
\newblock SDEdit: Guided Image Synthesis and Editing with Stochastic Differential Equations.
\newblock In \emph{The Tenth International Conference on Learning Representations, {ICLR} 2022, Virtual Event, April 25-29, 2022}. OpenReview.net.

\bibitem[{Nichol and Dhariwal(2021)}]{improved-ddpm}
Nichol, A.~Q.; and Dhariwal, P. 2021.
\newblock Improved Denoising Diffusion Probabilistic Models.
\newblock In Meila, M.; and Zhang, T., eds., \emph{Proceedings of the 38th International Conference on Machine Learning, {ICML} 2021, 18-24 July 2021, Virtual Event}, volume 139 of \emph{Proceedings of Machine Learning Research}, 8162--8171. {PMLR}.

\bibitem[{Russakovsky et~al.(2015)Russakovsky, Deng, Su, Krause, Satheesh, Ma, Huang, Karpathy, Khosla, Bernstein, Berg, and Fei{-}Fei}]{imagenet}
Russakovsky, O.; Deng, J.; Su, H.; Krause, J.; Satheesh, S.; Ma, S.; Huang, Z.; Karpathy, A.; Khosla, A.; Bernstein, M.~S.; Berg, A.~C.; and Fei{-}Fei, L. 2015.
\newblock ImageNet Large Scale Visual Recognition Challenge.
\newblock \emph{Int. J. Comput. Vis.}, 115(3): 211--252.

\bibitem[{Saharia et~al.(2022{\natexlab{a}})Saharia, Chan, Chang, Lee, Ho, Salimans, Fleet, and Norouzi}]{palette}
Saharia, C.; Chan, W.; Chang, H.; Lee, C.~A.; Ho, J.; Salimans, T.; Fleet, D.~J.; and Norouzi, M. 2022{\natexlab{a}}.
\newblock Palette: Image-to-Image Diffusion Models.
\newblock In Nandigjav, M.; Mitra, N.~J.; and Hertzmann, A., eds., \emph{{SIGGRAPH} '22: Special Interest Group on Computer Graphics and Interactive Techniques Conference, Vancouver, BC, Canada, August 7 - 11, 2022}, 15:1--15:10. {ACM}.

\bibitem[{Saharia et~al.(2022{\natexlab{b}})Saharia, Chan, Saxena, Li, Whang, Denton, Ghasemipour, Lopes, Ayan, Salimans, Ho, Fleet, and Norouzi}]{imagen}
Saharia, C.; Chan, W.; Saxena, S.; Li, L.; Whang, J.; Denton, E.~L.; Ghasemipour, S. K.~S.; Lopes, R.~G.; Ayan, B.~K.; Salimans, T.; Ho, J.; Fleet, D.~J.; and Norouzi, M. 2022{\natexlab{b}}.
\newblock Photorealistic Text-to-Image Diffusion Models with Deep Language Understanding.
\newblock In \emph{NeurIPS}.

\bibitem[{Saharia et~al.(2021)Saharia, Ho, Chan, Salimans, Fleet, and Norouzi}]{sr3}
Saharia, C.; Ho, J.; Chan, W.; Salimans, T.; Fleet, D.~J.; and Norouzi, M. 2021.
\newblock Image Super-Resolution via Iterative Refinement.
\newblock arXiv:2104.07636.

\bibitem[{Salimans and Ho(2022)}]{diff-distill}
Salimans, T.; and Ho, J. 2022.
\newblock Progressive Distillation for Fast Sampling of Diffusion Models.
\newblock In \emph{The Tenth International Conference on Learning Representations, {ICLR} 2022, Virtual Event, April 25-29, 2022}. OpenReview.net.

\bibitem[{Song, Meng, and Ermon(2021)}]{ddim}
Song, J.; Meng, C.; and Ermon, S. 2021.
\newblock Denoising Diffusion Implicit Models.
\newblock In \emph{9th International Conference on Learning Representations, {ICLR} 2021, Virtual Event, Austria, May 3-7, 2021}. OpenReview.net.

\bibitem[{Song and Ermon(2019)}]{multidenoise}
Song, Y.; and Ermon, S. 2019.
\newblock Generative Modeling by Estimating Gradients of the Data Distribution.
\newblock In Wallach, H.~M.; Larochelle, H.; Beygelzimer, A.; d'Alch{\'{e}}{-}Buc, F.; Fox, E.~B.; and Garnett, R., eds., \emph{Advances in Neural Information Processing Systems 32: Annual Conference on Neural Information Processing Systems 2019, NeurIPS 2019, December 8-14, 2019, Vancouver, BC, Canada}, 11895--11907.

\bibitem[{Song et~al.(2021)Song, Sohl{-}Dickstein, Kingma, Kumar, Ermon, and Poole}]{scoresde}
Song, Y.; Sohl{-}Dickstein, J.; Kingma, D.~P.; Kumar, A.; Ermon, S.; and Poole, B. 2021.
\newblock Score-Based Generative Modeling through Stochastic Differential Equations.
\newblock In \emph{9th International Conference on Learning Representations, {ICLR} 2021, Virtual Event, Austria, May 3-7, 2021}. OpenReview.net.

\bibitem[{Vincent(2011)}]{denoise-sm}
Vincent, P. 2011.
\newblock A Connection Between Score Matching and Denoising Autoencoders.
\newblock \emph{Neural Comput.}, 23(7): 1661--1674.

\bibitem[{Wang et~al.(2018)Wang, Yu, Wu, Gu, Liu, Dong, Qiao, and Loy}]{esrgan}
Wang, X.; Yu, K.; Wu, S.; Gu, J.; Liu, Y.; Dong, C.; Qiao, Y.; and Loy, C.~C. 2018.
\newblock {ESRGAN:} Enhanced Super-Resolution Generative Adversarial Networks.
\newblock In Leal{-}Taix{\'{e}}, L.; and Roth, S., eds., \emph{Computer Vision - {ECCV} 2018 Workshops - Munich, Germany, September 8-14, 2018, Proceedings, Part {V}}, volume 11133 of \emph{Lecture Notes in Computer Science}, 63--79. Springer.

\bibitem[{Wang, Yu, and Zhang(2023)}]{ddnm}
Wang, Y.; Yu, J.; and Zhang, J. 2023.
\newblock Zero-Shot Image Restoration Using Denoising Diffusion Null-Space Model.
\newblock In \emph{The Eleventh International Conference on Learning Representations, {ICLR} 2023, Kigali, Rwanda, May 1-5, 2023}. OpenReview.net.

\bibitem[{Zhang et~al.(2018{\natexlab{a}})Zhang, Isola, Efros, Shechtman, and Wang}]{lpips}
Zhang, R.; Isola, P.; Efros, A.~A.; Shechtman, E.; and Wang, O. 2018{\natexlab{a}}.
\newblock The Unreasonable Effectiveness of Deep Features as a Perceptual Metric.
\newblock In \emph{2018 {IEEE} Conference on Computer Vision and Pattern Recognition, {CVPR} 2018, Salt Lake City, UT, USA, June 18-22, 2018}, 586--595. Computer Vision Foundation / {IEEE} Computer Society.

\bibitem[{Zhang et~al.(2018{\natexlab{b}})Zhang, Li, Li, Wang, Zhong, and Fu}]{rcan}
Zhang, Y.; Li, K.; Li, K.; Wang, L.; Zhong, B.; and Fu, Y. 2018{\natexlab{b}}.
\newblock Image Super-Resolution Using Very Deep Residual Channel Attention Networks.
\newblock In Ferrari, V.; Hebert, M.; Sminchisescu, C.; and Weiss, Y., eds., \emph{European Conference on Computer Vision}, Lecture Notes in Computer Science, 294--310. Springer.

\bibitem[{Zhu et~al.(2023)Zhu, Zhang, Liang, Cao, Wen, Timofte, and Gool}]{diffpir}
Zhu, Y.; Zhang, K.; Liang, J.; Cao, J.; Wen, B.; Timofte, R.; and Gool, L.~V. 2023.
\newblock Denoising Diffusion Models for Plug-and-Play Image Restoration.
\newblock In \emph{{IEEE/CVF} Conference on Computer Vision and Pattern Recognition, {CVPR} 2023 - Workshops, Vancouver, BC, Canada, June 17-24, 2023}, 1219--1229. {IEEE}.

\end{thebibliography}

\clearpage
\onecolumn
\appendix

\setcounter{secnumdepth}{2}

\section{Extended derivations}
\subsection{Proof of the mean and variance-preserving property}
In this section, we will prove the mean and variance-preserving property of our conditional forward process as defined by \eqref{eq:ours-forward}. First, we will show that the VP-SDE preserves the data mean and variance in the zero-mean and unit-variance case.
\begin{lemma*}
    If the distribution of the input $\bm{x}(0)$ to the forward process of VP-SDE, as given in the following equation, has zero mean and unit variance, then the distribution of $\bm{x}(t)$ has zero mean and unit variance for any $t$.
    \begin{equation}
        \mathrm{d} \bm{x} = -\frac{1}{2} \beta(t) \bm{x} \mathrm{d} t + \sqrt{\beta(t)} \mathrm{d} \bm{w}
    \end{equation}
\end{lemma*}
\begin{proof}
    From the transition probability of VP-SDE,
    \begin{equation*}
        p_{0t}(\bm{x}(t) \mid \bm{x}(0)) = \mathcal{N}(\bm{x}(t); \sqrt{\alpha(t)} \bm{x}(0), (1 - \alpha(t)) \bm{I})
    \end{equation*}
    It follows naturally that
    \begin{align*}
        \mathbb{E}[\bm{x}(t)] &= \sqrt{\alpha(t)} \bm{x}(0) = 0 \\
        \text{Var}[\bm{x}(t)] &= \text{Var}\left[\bm{x}(t) - \sqrt{\alpha(t)} \bm{x}(0)\right] + \text{Var}\left[\sqrt{\alpha(t)} \bm{x}(0)\right] \\
        &= (1 - \alpha(t)) + \alpha(t) = 1
    \end{align*}
\end{proof}

For our forward SDE \eqref{eq:ours-forward}:
\begin{equation*}
    \mathrm{d} \bm{x} = -\frac{1}{2} \beta(t) (\bm{x} - \bm{\mu}(\bm{y})) \mathrm{d} t + \sqrt{\beta(t) \sigma^2(\bm{y})} \mathrm{d} \bm{w}
\end{equation*}
We show that it preserves the data mean and variance of the conditional distribution $p(\bm{x} \mid \bm{x})$. We consider $\bm{y}$ to be a fixed value in the following derivation.
Define $\bm{z}$ as follows:
\begin{equation*}
    \bm{z} = \frac{1}{\sqrt{\sigma^2(\bm{y})}}(\bm{x} - \bm{\mu}(\bm{y}))
\end{equation*}
then $\bm{z}$ follows the forward equation of VP-SDE, as shown below:
\begin{align*}
    \mathrm{d} \bm{z} &= \frac{1}{\sigma^2(\bm{y})} \mathrm{d} \bm{x} \\
    &= -\frac{\beta(t)}{2 \sqrt{\sigma^2(\bm{y})}}(\bm{x}  - \bm{\mu}(\bm{y})) \mathrm{d} t + \sqrt{\beta(t)} \mathrm{d} \bm{w} \\
    &= -\frac{1}{2} \beta(t) \bm{z} \mathrm{d} t + \sqrt{\beta(t)} \mathrm{d} \bm{w}
\end{align*}

Since $\bm{z}(0)$ has zero mean and unit variance by its definition, $\bm{z}(t)$ has zero mean and unit variance for any $t$.
From $\bm{x} = \sqrt{\sigma^2(\bm{y})} \bm{z} + \bm{\mu}(\bm{y})$, we have
\begin{align*}
    \mathbb{E}[\bm{x}(t) \mid \bm{y}] &= \sqrt{\sigma^2(\bm{y})} \mathbb{E}[\bm{z}(t) \mid \bm{y}] + \bm{\mu}(\bm{y}) = \bm{\mu}(\bm{y}) \\
    \text{Var}[\bm{x}(t) \mid \bm{y}] &= \sigma^2(\bm{y}) \text{Var}[\bm{z}(t) \mid \bm{y}] = \sigma^2(\bm{y})
\end{align*}
This demonstrates that our forward SDE preserves mean and variance of $p(\bm{x} \mid \bm{y})$ during the transform.

\subsection{Derivation of the score matching loss}
With $\bm{y}$ set to a fixed value, our forward SDE is a case of the general forward SDE \eqref{eq:general-forward} with
\begin{align*}
    \bm{f}(\bm{x}, t) &= -\frac{1}{2} \beta(t) (\bm{x} - \bm{\mu}(\bm{y})) \\
    g(t) &= \sqrt{\beta(t) \sigma^2(\bm{y})}
\end{align*}

It is possible to derive the transition probability of our forward SDE from the transition probability of $\bm{z}$, as defined in the previous section:
\begin{align*}
    p_{0t}(\bm{z}(t) \mid \bm{z}(0)) &= \mathcal{N}(\bm{z}(t); \sqrt{\alpha(t)} \bm{z}(0), (1 - \alpha(t)) \bm{I}) \\
    p_{0t}(\bm{x}(t) \mid \bm{x}(0), \bm{y}) &= \mathcal{N}(\bm{x}(t); \sqrt{\alpha(t)} (\bm{x}(0) - \bm{\mu}(\bm{y})) + \bm{\mu}(\bm{y}), (1 - \alpha(t)) \sigma^2(\bm{y}) \bm{I}) \\
    &= \mathcal{N}(\bm{x}(t); \hat{\bm{\mu}}(\bm{x}(0), \bm{y}, t), \hat{\sigma}^2(\bm{y}, t) \bm{I})
\end{align*}
where $\hat{\bm{\mu}}$ and $\hat{\sigma}^2$ are defined as in the main text:
\begin{align*}
    \hat{\bm{\mu}}(\bm{x}, \bm{y}, t) &= \sqrt{\alpha(t)} (\bm{x} - \bm{\mu}(\bm{y})) + \bm{\mu}(\bm{y}) \\
    \hat{\sigma}^2(\bm{y}, t) &= (1 - \alpha(t)) \sigma^2(\bm{y})
\end{align*}

To obtain the score matching loss for our score predictor $\bm{s}_{\theta}$, we substitute the definitions of $\bm{f}(\bm{x}, t)$ and $g(t)$ in \eqref{eq:general-scoreloss}:
\begin{align*}
    \mathcal{L}_{\text{score}} &= \mathbb{E}_{(\bm{x}, \bm{y}) \sim \mathcal{D}} \mathbb{E}_{t} \mathbb{E}_{\bm{x}(t)} \left[ \norm*{\bm{s}_{\theta}(\bm{x}(t), \bm{y}, t) - \nabla_{\bm{x}(t)} \log p_{0t}(\bm{x}(t) \mid \bm{x}, \bm{y})}_2^2 \right] \\
    &= \mathbb{E}_{(\bm{x}, \bm{y}) \sim \mathcal{D}} \mathbb{E}_{t} \mathbb{E}_{\bm{x}(t)} \left[ \norm*{\bm{s}_{\theta}(\bm{x}(t), \bm{y}, t) + \frac{1}{\hat{\sigma}^2(\bm{y}, t)} (\bm{x}(t) - \hat{\bm{\mu}}(\bm{x}, \bm{y}, t))}_2^2 \right] \\
    &= \mathbb{E}_{(\bm{x}, \bm{y}) \sim \mathcal{D}} \mathbb{E}_{t}  \mathbb{E}_{\bm{\epsilon} \sim \mathcal{N}(\bm{0}, \bm{I})} \left[ \norm*{\bm{s}_{\theta}\left(\hat{\bm{\mu}}(\bm{x}, \bm{y}, t) + \sqrt{\hat{\sigma}^2(\bm{y}, t)} \bm{\epsilon}, \bm{y}, t \right) + \frac{1}{\sqrt{\hat{\sigma}^2(\bm{y}, t)}} \bm{\epsilon}}_2^2 \right]
\end{align*}

\subsection{Derivation of the probability flow}
By substituting $\bm{f}(\bm{x}, t)$ and $g(t)$ into the general probability flow \eqref{eq:general-ode}, we obtain the form of our probability flow \eqref{eq:ours-ode}.
\begin{equation*}
    \mathrm{d} \bm{x} = \left[ -\frac{1}{2} \beta(t) (\bm{x} - \bm{\mu}(\bm{y})) - \beta(t) \sigma^2(\bm{y}) \nabla_{\bm{x}} \log p_t(\bm{x} \mid \bm{y}) \right] \mathrm{d} t + \sqrt{\beta(t) \sigma^2(\bm{y})} \mathrm{d} \bm{w}
\end{equation*}
Since we consider $\bm{y}$ to be a fixed value in above, the distribution recovered by our probability flow is the conditional distribution $p(\bm{x} \mid \bm{y})$ as opposed to $p(\bm{x})$.

\subsection{Derivation of hybrid-parametrizations}
We show that the $\epsilon$-parametrization and $x_0$-parametrization can be expressed as follows:
\begin{align*}
    \bm{s}_{\theta,\epsilon}(\bm{x}, \bm{y}, t) &= -\frac{1}{\sqrt{\hat{\sigma}^2(\bm{y}, t)}} \bm{\epsilon}_{\theta}(\bm{x}, \bm{y}, t) \\
    \bm{s}_{\theta,x_0}(\bm{x}, \bm{y}, t) &= -\frac{1}{\hat{\sigma}^2(\bm{y})} \left( \bm{x} - \hat{\bm{\mu}}(\bm{x}_{0\theta}(\bm{x}, \bm{y}, t), \bm{y}, t) \right)
\end{align*}

We substitute their definitions into the score matching loss $\mathcal{L}_{\text{score}}$:
\begin{align*}
    \mathcal{L}_{\text{score},\epsilon} &= \mathbb{E}_{(\bm{x}, \bm{y}) \sim \mathcal{D}} \mathbb{E}_{t}  \mathbb{E}_{\bm{\epsilon} \sim \mathcal{N}(\bm{0}, \bm{I})} \left[ \norm*{\bm{s}_{\theta,\epsilon}\left(\hat{\bm{\mu}}(\bm{x}, \bm{y}, t) + \sqrt{\hat{\sigma}^2(\bm{y}, t)} \bm{\epsilon}, \bm{y}, t \right) + \frac{1}{\sqrt{\hat{\sigma}^2(\bm{y}, t)}} \bm{\epsilon}}_2^2 \right] \\
    &= \mathbb{E}_{(\bm{x}, \bm{y}) \sim \mathcal{D}} \mathbb{E}_{t} \mathbb{E}_{\bm{\epsilon} \sim \mathcal{N}(\bm{0}, \bm{I})} \left[ \frac{1}{\sigma^2(\bm{y}, t)} \norm*{\bm{\epsilon}_{\theta} \left(\hat{\bm{\mu}}(\bm{x}, \bm{y}, t) + \sqrt{\hat{\sigma}^2(\bm{y}, t)} \bm{\epsilon}, \bm{y}, t \right) - \bm{\epsilon}}_2^2 \right] \\
    \mathcal{L}_{\text{score},x_0} &= \mathbb{E}_{(\bm{x}, \bm{y}) \sim \mathcal{D}} \mathbb{E}_{t}  \mathbb{E}_{\bm{\epsilon} \sim \mathcal{N}(\bm{0}, \bm{I})} \left[ \norm*{\bm{s}_{\theta,x_0}\left(\hat{\bm{\mu}}(\bm{x}, \bm{y}, t) + \sqrt{\hat{\sigma}^2(\bm{y}, t)} \bm{\epsilon}, \bm{y}, t \right) + \frac{1}{\sqrt{\hat{\sigma}^2(\bm{y}, t)}} \bm{\epsilon}}_2^2 \right] \\
    &= \mathbb{E}_{(\bm{x}, \bm{y}) \sim \mathcal{D}} \mathbb{E}_{t} \mathbb{E}_{\bm{\epsilon} \sim \mathcal{N}(\bm{0}, \bm{I})} \left[ \frac{\alpha(t)}{[\hat{\sigma}^2(\bm{y}, t)]^2} \norm*{\bm{x}_{0\theta} \left(\hat{\bm{\mu}}(\bm{x}, \bm{y}, t) + \sqrt{\hat{\sigma}^2(\bm{y}, t)} \bm{\epsilon}, \bm{y}, t \right) - \bm{x}}_2^2 \right]
\end{align*}

It can be seen that $\epsilon$-parametrization and $x_0$-parametrization learns the noise component $\bm{\epsilon}$ and the clean data component $\bm{x}$ respectively. By reweighting terms for different $t$, these losses can be combined into the score matching loss for our hybrid-parametrization:
\begin{gather*}
\begin{split}
    \mathcal{L}_{\text{score}} &= \mathbb{E}_{(\bm{x}, \bm{y}) \sim \mathcal{D}} \mathbb{E}_{t} \mathbb{E}_{\bm{\epsilon} \sim \mathcal{N}(\bm{0}, \bm{I})} \\
    &\mathrel{\phantom{=}} \quad \left[ \norm*{\bm{\epsilon}_{\theta}(\bm{x}_t, \bm{y}, t) - \bm{\epsilon}}_2^2 + \norm*{\bm{x}_{0\theta}(\bm{x}_t, \bm{y}, t) - \bm{x}}_2^2 \right]
\end{split}
    \shortintertext{where}
    \bm{x}_t = \hat{\bm{\mu}}(\bm{x}, \bm{y}, t) + \sqrt{\hat{\sigma}^2(\bm{y}, t)} \bm{\epsilon}
\end{gather*}

\clearpage
\section{Details of model training}
The score predictor architectures for our method are provided in Table \ref{tab:training-details}, where ``Blocks'' is the number of residual blocks in each resolution stage of the U-Net architecture, ``Channels'' is the dimension of layers in each stage, and ``LR Blocks'' is the number of RRDB blocks in the LR feature extractor. All models are trained with the Adam optimizer, and the learning rate starts at $10^{-4}$ and are gradually reduced to $10^{-5}$ over the training procedure.

\begin{table}[h]
    \centering
    \begin{tabular}{cccc}
        \toprule
        Dataset & Blocks & Channels & LR Blocks \\
        \midrule
        DIV2K & [2, 2, 2, 2] & [64, 128, 256, 256] & 23 \\
        ImageNet & [2, 2, 2, 2] & [64, 128, 256, 256] & 16 \\
        CelebA & [3, 3, 3, 3] & [64, 128, 256, 256] & 8 \\
        \bottomrule
    \end{tabular}
    \caption{Score predictor architectures of our method on each dataset.}
    \label{tab:training-details}
\end{table}

\section{Extended experimental results}
\subsection{Effects of different parametrizations}
To further demonstrate the effect of different parametrizations, we provide the LPIPS metrics during training in Figure \ref{fig:appendix-ablation-param}. The metrics are produced by comparing $160 \times 160$ image patches from DIV2K validation set with super-resolution results generated using probability flow sampling. It can be seen that the hybrid parametrization achieves the best LPIPS results over the whole training process.

\begin{figure}[h]
    \centering
    \includegraphics[width=0.5\linewidth]{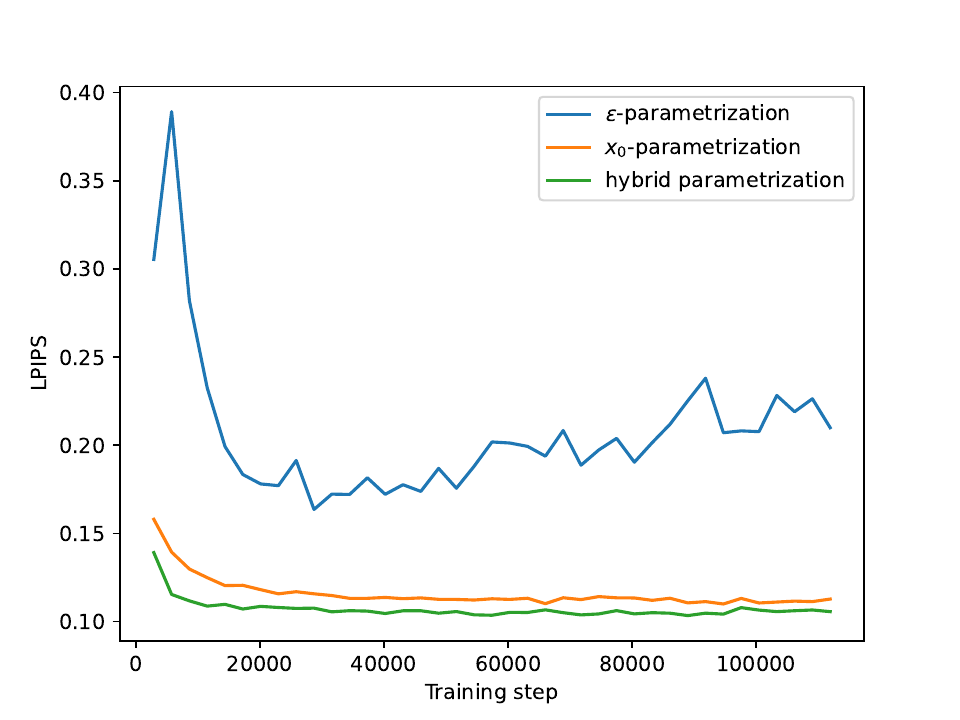}
    \caption{The LPIPS metrics on DIV2K validation set during training for different parametrizations. Note that the metrics are measured on $160 \times 160$ patches and do not correspond exactly to the main results in the paper.}
    \label{fig:appendix-ablation-param}
\end{figure}

\subsection{Effects of the image quality loss}
To further demonstrate the effect of the image quality loss, we adapt other diffusion-based super-resolution methods to use this loss as well. For SRDiff and SR3, since they use discrete-time diffusion models as opposed to the SDE-based diffusion model in our method, we compute the image quality loss using their standard sampling process instead. Additionally, we also test the performance of our method when the image quality loss is computed with stochastic sampling instead of probability flow sampling.

The results are shown in Table \ref{tab:appendix-ablation-finetune}.
It can be seen that $\mathcal{L}_{\text{quality}}$ significantly improves the performance of various diffusion-based super-resolution methods. This demonstrates the general usefulness of the image quality loss. Furthermore, it is more advantageous to compute $\mathcal{L}_{\text{quality}}$ with probability flow in our method, which is likely due to probability flow sampling being easier to optimize than stochastic sampling. It is also worth noting that, since the computation of the image quality loss requires generating new images, our method benefits the most from the loss as it has the most efficient sampling process.

\begin{table}[ht]
    \centering
    \setlength\tabcolsep{8pt}
    \begin{tabular}{cccc}
        \toprule
        Method & LPIPS\textdownarrow & PSNR\textuparrow & SSIM\textuparrow \\
        \midrule
        SRDiff & 0.136 & 27.41 & 0.79 \\
        SRDiff++ & 0.126 & 27.74 & 0.79 \\
        \midrule
        SR3 & 0.175 & 25.90 & 0.75 \\
        SR3++ & 0.167 & 25.20 & 0.72 \\
        \midrule
        Ours w/o $\mathcal{L}_{\text{quality}}$ & 0.120 & 27.26 & 0.76 \\
        Ours w/o prob. flow & 0.112 & 27.80 & 0.78 \\
        Ours & 0.108 & 28.03 & 0.79 \\
        \bottomrule
    \end{tabular}
    \caption{Effect of adding the image quality loss to various diffusion-based super-resolution methods. All results are measured on DIV2K validation set. ``SRDiff++'' and ``SR3'' denotes improved variants of SRDiff and SR3 respectively.}
    \label{tab:appendix-ablation-finetune}
\end{table}

\subsection{Extended visual results}
We provide extended visualization of super-resolution results in Figure \ref{fig:appendix-results-general} and Figure \ref{fig:appendix-results-face}, to demonstrate the performance of our method.

\begin{figure*}[ht]
    \centering
    \begin{tabular}{cccccc}
        \multirow{4}{*}[0.095\linewidth]{\includegraphics[width=0.25\linewidth]{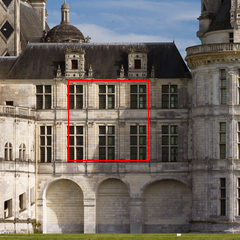}}
        & \includegraphics[width=0.11\linewidth]{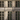}
        & \includegraphics[width=0.11\linewidth]{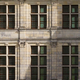}
        & \includegraphics[width=0.11\linewidth]{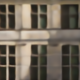}
        & \includegraphics[width=0.11\linewidth]{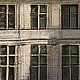}
        & \includegraphics[width=0.11\linewidth]{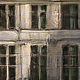} \\
        & LR & Ground Truth & RRDB & ESRGAN & SRFlow \\
        & \includegraphics[width=0.11\linewidth]{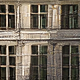}
        & \includegraphics[width=0.11\linewidth]{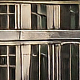}
        & \includegraphics[width=0.11\linewidth]{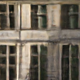}
        & \includegraphics[width=0.11\linewidth]{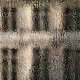}
        & \includegraphics[width=0.11\linewidth]{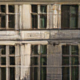} \\
        & HCFlow & SR3 & SRDiff & IR-SDE & Ours \\
        \multirow{4}{*}[0.095\linewidth]{\includegraphics[width=0.25\linewidth]{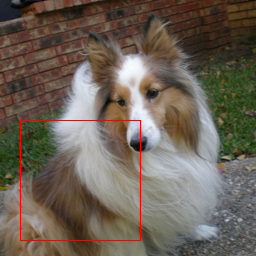}}
        & \includegraphics[width=0.11\linewidth]{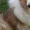}
        & \includegraphics[width=0.11\linewidth]{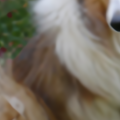}
        & \includegraphics[width=0.11\linewidth]{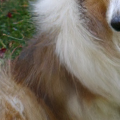}
        & \includegraphics[width=0.11\linewidth]{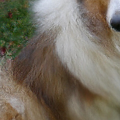}
        & \includegraphics[width=0.11\linewidth]{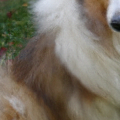} \\
        & LR & RRDB & ESRGAN & SRFlow & HCFlow \\
        & \includegraphics[width=0.11\linewidth]{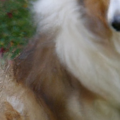}
        & \includegraphics[width=0.11\linewidth]{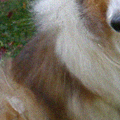}
        & \includegraphics[width=0.11\linewidth]{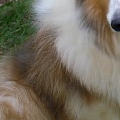}
        & \includegraphics[width=0.11\linewidth]{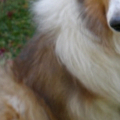}
        & \includegraphics[width=0.11\linewidth]{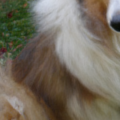} \\
        & SR3 & SRDiff & GDP & DiffPIR & Ours
    \end{tabular}
    \caption{Extended visual results of general image super-resolution. The first two rows are results on DIV2K validation set. The last two rows are results on ImageNet dev set.}
    \label{fig:appendix-results-general}
\end{figure*}

\begin{figure*}[ht]
    \centering
    \newlength{\mytablecellsize}
    \setlength\tabcolsep{1.5pt}
    \setlength{\mytablecellsize}{0.95\linewidth}
    \begin{tabular}{ccccccc|cc}
        \includegraphics[width=0.11\mytablecellsize]{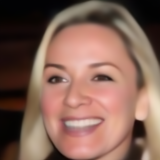} &
        \includegraphics[width=0.11\mytablecellsize]{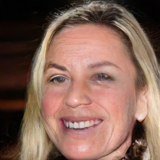} &
        \includegraphics[width=0.11\mytablecellsize]{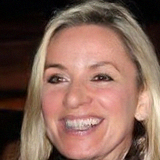} &
        \includegraphics[width=0.11\mytablecellsize]{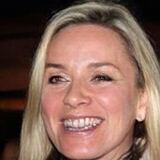} &
        \includegraphics[width=0.11\mytablecellsize]{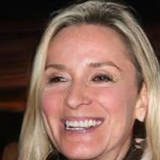} &
        \includegraphics[width=0.11\mytablecellsize]{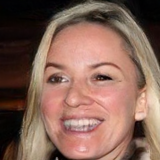} &
        \includegraphics[width=0.11\mytablecellsize]{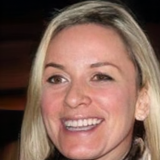} &
        \includegraphics[width=0.11\mytablecellsize]{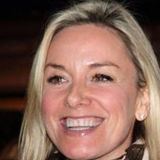} &
        \includegraphics[width=0.11\mytablecellsize]{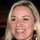} \\
        \includegraphics[width=0.11\mytablecellsize]{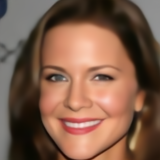} &
        \includegraphics[width=0.11\mytablecellsize]{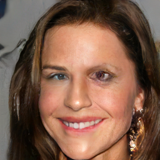} &
        \includegraphics[width=0.11\mytablecellsize]{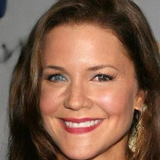} &
        \includegraphics[width=0.11\mytablecellsize]{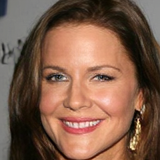} &
        \includegraphics[width=0.11\mytablecellsize]{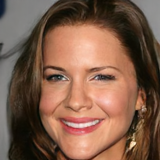} &
        \includegraphics[width=0.11\mytablecellsize]{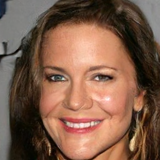} &
        \includegraphics[width=0.11\mytablecellsize]{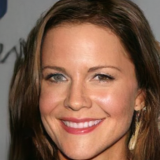} &
        \includegraphics[width=0.11\mytablecellsize]{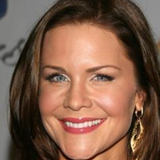} &
        \includegraphics[width=0.11\mytablecellsize]{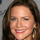} \\
        \includegraphics[width=0.11\mytablecellsize]{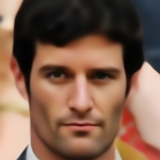} &
        \includegraphics[width=0.11\mytablecellsize]{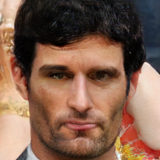} &
        \includegraphics[width=0.11\mytablecellsize]{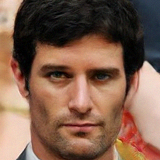} &
        \includegraphics[width=0.11\mytablecellsize]{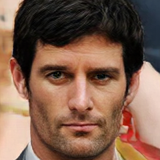} &
        \includegraphics[width=0.11\mytablecellsize]{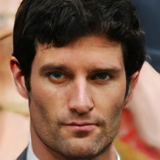} &
        \includegraphics[width=0.11\mytablecellsize]{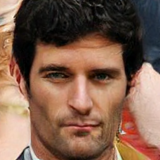} &
        \includegraphics[width=0.11\mytablecellsize]{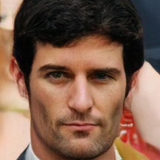} &
        \includegraphics[width=0.11\mytablecellsize]{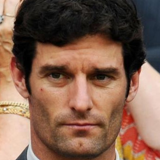} &
        \includegraphics[width=0.11\mytablecellsize]{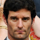} \\
        RRDB & ESRGAN & SRFlow & HCFlow & SR3 & SRDiff & Ours & Ground Truth & LR Input
    \end{tabular}
    \caption{Visual results of face image super-resolution on CelebA test set.}
    \label{fig:appendix-results-face}
\end{figure*}

\end{document}